\documentclass[letterpaper]{article} 
\usepackage{aaai2026}
\usepackage{times}  
\usepackage{helvet}  
\usepackage{courier}  
\usepackage[hyphens]{url}  
\usepackage{graphicx} 
\usepackage{natbib}  
\usepackage{caption} 
\usepackage{amsmath}
\usepackage{amssymb}
\usepackage{booktabs}
\usepackage{amsthm}
\usepackage{subcaption}
\usepackage{paralist}
\usepackage{enumitem}
\usepackage{todonotes}
\usepackage{easyReview}
\presetkeys{todonotes}{inline}{}
\usepackage{multirow}
\usepackage{standalone}
\standaloneconfig{mode=buildnew}
\usepackage{pgfplots}
\pgfplotsset{compat=1.18}
\usepackage{tikz}

\usepackage{algorithm}
\usepackage{algorithmicx}
\usepackage[noend]{algpseudocode}

\usepackage{xstring}
\usepackage{xspace}
\usepackage{tablefootnote}

\newtheorem{theorem}{Theorem}
\newtheorem{defn}{Definition}

\newcommand{\define}[1]{\emph{#1}}
\newcommand{\bracketcite}[1]{\citeauthor{#1} (\citeyear{#1})}
\newcommand{\tuple}[1]{\ensuremath{\langle #1 \rangle}}
\newcommand{\facttuple}[2]{\ensuremath{\langle #1, #2 \rangle}}
\newcommand{\factv}[2]{\ensuremath{#1 \mapsto #2}}
\newcommand{\fact}{\ensuremath{p}}
\newcommand{\pre}{\ensuremath{\mathit{pre}}}
\newcommand{\adds}{\ensuremath{\mathit{add}}}
\newcommand{\del}{\ensuremath{\mathit{del}}}

\newcommand{\cost}{\ensuremath{\mathit{c}}}
\newcommand{\utility}{\ensuremath{u}}
\newcommand{\costbound}{\ensuremath{b}}
\newcommand{\eff}{\ensuremath{\mathit{eff}}}
\newcommand{\pos}{\ensuremath{\mathit{pos}}}
\newcommand{\vars}{\ensuremath{\mathit{vars}}}
\newcommand{\skipa}{\ensuremath{\mathit{skip}}}
\newcommand{\sasplus}{SAS$^+$}
\newcommand{\cbpr}{\ensuremath{\Delta}}
\newcommand{\PiSkip}{\Pi^{\textit{s1}}_{\textit{a1}}}
\newcommand{\PiMR}{\Pi^{\textit{r}}}

\newcommand{\events}{\ensuremath{E}}
\newcommand{\event}{\ensuremath{e}}
\newcommand{\preset}{\ensuremath{P}}
\newcommand{\lastevent}{\ensuremath{\sigma}}

\newcommand{\A}{\mathcal{A}}
\newcommand{\I}{\mathcal{I}}
\newcommand{\G}{\mathcal{G}}
\newcommand{\V}{\mathcal{V}}
\newcommand{\D}{\mathcal{D}}

\newcommand{\OSP}{\ensuremath{\text{OSP}}\xspace}
\newcommand{\ILP}{\ensuremath{\text{ILP}}\xspace}
\newcommand{\ILPmt}{\ensuremath{\text{ILP}_{\text{mt}}}\xspace}
\newcommand{\ILPst}{\ensuremath{\text{ILP}_{\text{st}}}\xspace}
\newcommand{\ILPplus}{\ensuremath{\text{ILP}^+}\xspace}
\newcommand{\ILPplusst}{\ensuremath{\text{ILP}^+_{\text{st}}}\xspace}
\newcommand{\ILPplusstall}{\ensuremath{\text{ILP}^+_{\text{st,all}}}\xspace}

\newcommand{\floortile}{\textit{Floortile}}
\newcommand{\tidybot}{\textit{Tidybot}}
\newcommand{\openstacks}{\textit{OpenStacks}}
\newcommand{\transport}{\textit{Transport}}
\newcommand{\visitall}{\textit{Visitall}}
\newcommand{\elevators}{\textit{Elevators}}
\newcommand{\parcprinter}{\textit{ParcPrinter}}
\newcommand{\woodworking}{\textit{Woodworking}}

\newcommand{\scanalyzer}{\textit{Scanalyzer}}

\copyrighttext{This article is an extended version of a paper accepted at the International Conference on Automated Planning and Scheduling (ICAPS 2026). Copyright \copyright\ 2026, Association for the Advancement of Artificial Intelligence (www.aaai.org). All rights reserved. 
}

\title{Finding Optimal Cost-Bounded Plan Reductions: Refined Model}
\author{
    Martha Del Toro,
    Raquel Fuentetaja,
    Angel García-Olaya
}
\affiliations{
    Computer Science and Engineering Department, Universidad Carlos III de Madrid, Spain\\
    mdeltoro@pa.uc3m.es, \{rfuentet, agolaya\}@inf.uc3m.es
}

\begin{document}

\maketitle

\begin{abstract}
  In some real applications a plan may later become unfeasible due to newly imposed budget constraints, yet, at the same time, using only the original actions of the plan and their order is mandatory.
  In this paper, we study the problem of extracting, from a precomputed plan, a valid subplan that maximizes utility while respecting a cost bound.
  Each goal is given a utility value and the plan is reduced by removing actions that support low-utility goals, while preserving both executability and the original action order.
  We show the decision variant is NP-complete and propose two exact methods to solve it: one via oversubscription planning (OSP) and another via Integer Linear Programming (ILP).
  This paper extends our previous work published at ICAPS 2026 \cite{deltoro-et-al-icaps2026}. While the core framework remains as introduced there, we further introduce a refined ILP formulation that significantly decreases the model size and improves computational efficiency.
\end{abstract}

\section{Introduction}
Automated planning systems generate plans that are executed under resource constraints that can change dynamically. When a cost limit is imposed, the original plan may become infeasible, and some goals may need to be dropped to satisfy the new cost constraints.
Keeping a subplan rather than replanning is essential in applications where the original plan has already been externally validated or relies on pre-committed resources. In such cases, the overhead of approving a new plan can be high, and full reorganization is often too costly or too slow.
For instance, in space operations, command sequences must meet strict safety standards, and operators may prefer to skip specific tasks rather than upload an unverified sequence that could risk the mission.
Similarly, in industrial workflows, operators follow established procedures and fixed schedules, making task omission simpler than reorganizing the entire workflow.
With that in mind, we study the problem of extracting a valid subplan from a precomputed plan that maximizes utility while respecting a cost bound.
This problem is closely related to \emph{oversubscription planning} (OSP)~\cite{smith-icaps2004,domshlak-mirkis-jair2015, katz-et-al-icaps2019, garciaolaya-et-al-aij2021, speck-katz-aaai2021, katz-keyder-aaai2022}, with the difference that we operate on a precomputed plan and allow only action deletion, preserving the plan structure.
It also differs from traditional plan-repair techniques~\cite{fox-et-al-icaps2006}, which typically aim to restore all original goals after unexpected changes in the environment; in our case, some goals are deliberately dropped to satisfy a new budget.

We show that even the decision version of the problem is NP-complete.
To solve it, we build on techniques for identifying and removing redundant actions from plans~\cite{fink-yang-cscsi1992,nakhost-muller-icaps2010,chrpa-et-al-ictai2012,chrpa-et-al-icaps2012,balyo-et-al-socs2014, salerno-et-al-jair2025}, specifically on the one introduced by \bracketcite{salerno-et-al-jair2025}, and propose two new exact methods: one based on compiling the problem into an OSP instance, and another based on formulating it as an Integer Linear Program (ILP), using two different models.

Our empirical evaluation shows that the ILP approach outperforms OSP for solving this problem. Furthermore, the refined model introduced in this paper significantly improves upon our previous formulation.
The remainder of this paper is organized as follows. Sections \ref{sec:background} and \ref{sec:problem-definition-and-models} (covering the background, formal problem definition, and base formulation) are basically the ones in the original ICAPS paper \cite{deltoro-et-al-icaps2026}. Section \ref{sec:refined-model} is entirely new and introduces our refined ILP formulation. Section \ref{sec:experiments} presents an expanded experimental evaluation, including additional experiments that compare the two ILP formulations. Finally, Section \ref{sec:conclusion} presents the conclusions and future work.

\section{Background}
\label{sec:background}

We consider classical planning tasks with action costs in the \sasplus formalism \cite{backstrom-nebel-compint1995}.
\begin{defn}[\textbf{Planning task}]
  \label{defn:planning-task}
  A planning task is a tuple $\Pi=\langle \V,\A, \I, \G,  \cost \rangle$, where:
  \begin {itemize}

  \item $\V$ is a set of  \define{state variables}, each with a finite \define{domain} $\D(v)$.
  A pair $\facttuple{v}{d}$ such that $v \in V$ and $d \in \D(v)$ is a \define{fact}, denoted as $\factv{v}{d}$. $F$ is the set of all facts.
  A \define{partial state} $s$ maps a subset of variables $\vars(s) \subseteq \V$ to values in their domains.
  We often interpret partial states as sets of facts.
  The value of variable $v \in \vars(s)$ in partial state $s$ is denoted as $s[v] \in \D(v)$.
  A partial state that assigns values to all variables ($\vars(s) = \V$) is a \define{state}.

  \item  $\A$ is a finite set of {\it actions}.
  Each \define{action} $a \in \A$ is a pair $\tuple{pre(a), \eff(a)}$, where $\pre(a)$ and $\eff(a)$ are both partial states defining the \define{precondition} and the \define{effect} of $a$, respectively.
  An action $a \in \A$ is applicable in state $s$ iff $\pre(a) \subseteq s$.
  Applying action $a$ in $s$ yields the successor state $s[\![a]\!]$, with $s[\![a]\!][v] = \eff(a)[v]$ if $v \in \vars(\eff(a))$ and as $s[\![a]\!][v]= s[v]$ otherwise.

  \item  $\I$ is the \define{initial state}.

  \item  $\G$ is a partial state describing the \define{goal condition}.

  \item  $\cost$ is a non-negative \define{cost function} defining the cost of each  action $\cost: \A \to \mathbb{N}_0$.

  \end{itemize}
\end{defn}

A \define{solution} or \define{plan} for $\Pi$ is an action sequence $\pi = \tuple{a_1, \dots, a_n}$ that, when applied in succession starting from the initial state, induces a state sequence $\mathcal{S_\pi}=\langle s_0,\dots,s_n \rangle$ such that $s_0 = \I$, $\G \subseteq s_n$, and for each $i$ with $1 \leq i \leq n$, $a_i$ is applicable in $s_{i-1}$, and $s_i = s_{i-1}[\![a_i]\!]$.
We denote the \define{length} of plan $\pi = \tuple{a_1, \dots, a_n}$ as $|\pi| = n$.
The \define{cost} of a plan $\pi$ is $\cost(\pi) = \sum_{i=1}^n \cost(a_i)$.
A plan is \define{optimal} if there is no cheaper plan.
We slightly abuse notation and let $a_i \in \pi$ denote that action $a_i$ occurs at position $i$ of $\pi$.
A \define{plan reduction} of a plan $\pi$ for a planning task $\Pi$ is a subsequence of $\pi$ that is also a plan for $\Pi$~\cite{salerno-et-al-kr2023}.
Techniques for filtering redundant actions from a plan aim to find \define{perfectly justified} plan reductions for it, i.e., subplans without redundant actions.
A plan $\pi$ for $\Pi$ is \define{perfectly justified} if it admits no plan reduction $\rho$ of \(\pi\) such that $|\rho| < |\pi|$.
A \define{Minimal Reduction (MR)} of a plan is a perfectly justified plan reduction of it, with no other perfectly justified plan reduction having a lower cost~\cite{nakhost-muller-icaps2010,balyo-et-al-socs2014}.
\bracketcite{salerno-et-al-jair2025} encode the problem of finding an MR of a given plan as a classical planning task.
Their best-performing variant is the \emph{skip one, apply one} compilation, where each action in the original plan is explicitly represented and, at each position, the corresponding action can either be applied or skipped.
Since this technique is relevant to our setting, we summarize it next.
They define $F_\pi= R_\pi \cup W_\pi$, where $R_\pi$ is the set of facts that appear in a precondition of the actions in the plan $\pi$ or in the goal, and $W_\pi$ is the set of facts that appear in the effects of the actions or the initial state.
Then, they define the function $\tau$ that, given a fact $\factv{v}{d}$, maps the fact to itself if it is in  $R_\pi$ and to a new fact named the \emph{irrelevant fact} ($\factv{v}{\theta}$) otherwise, where $\theta$ denotes that $v$ takes a value that is not used afterward (it is not a precondition for any subsequent action and is not in the goal).
The encoding for solving the minimal reduction problems is a planning task $\PiMR=\tuple{\V',\A', \I', \G', \cost'}$, with facts $F'_\pi = \{\tau(\factv{v}{d}) \mid \factv{v}{d} \in F_\pi \}$, referred to as the \emph{relevant facts} for $\pi$; where:\footnote{$\PiMR$ is denoted as $\PiSkip$ in \bracketcite{salerno-et-al-jair2025}.}
\begin{itemize}
  \item $\V'$ retains only those variables in $\V$ with multiple values in $F'_\pi$ and contains a new variable $\pos$ to track action positions in the original plan, $\D(\pos)=\{0,\dots,n\}$.

  \item $\A'$ contains, for each plan action, an \emph{apply} action $a_i$ and a \emph{skip} action $\skipa_i$.
        Apply actions have the same preconditions as the original action plus an additional precondition enforcing that the current position is the preceding one.
        Their effects are those of the original action, mapped by $\tau$ to keep only relevant values, plus an additional effect that updates $\pos$.
        Skip actions simply increment $\pos$.

  \item $\I'$ contains the facts relevant for the plan $(\I \cap R_\pi)$, plus facts setting variables in $\V'$ with an irrelevant initial value to $\theta$, and one fact setting $\pos$ to $0$.

  \item $\G'$ contains the original goals and requires the {\it pos} variable to be at the end of the original plan.

  \item $\cost'$ assigns the cost of the corresponding original action to the \emph{apply} actions and a cost of zero to the \emph{skip} actions.\footnote{The authors apply a cost scaling to deal with zero-cost actions that we do not consider in this paper.}

\end{itemize}

\begin{defn}[\textbf{OSP task}]
  \label{defn:osp-task}
  An \define{oversubscription planning task} is a tuple $\Pi_{osp} = \tuple{\V, \A, \I, \cost, \utility, \costbound}$, where $\V$, $\A$, $\I$ and $\cost$ are as in Definition \ref{defn:planning-task}, with $F$ as the set of all facts; $\utility: F \to \mathbb{N}_0$ is a fact utility function and $\costbound \in \mathbb{N}_0$ is a cost bound.
\end{defn}

A plan $\pi$ for an OSP task $\Pi_{osp}$ is a sequence of operators applicable in $\I$ that yields states $s_1, \dots, s_n$, such that its cumulative cost is less than or equal to the bound, $\cost(\pi) \leq \costbound$.
The utility of a plan $\utility(\pi)$ is the utility of the end state induced by $\pi$, $\utility(\pi) = \utility(s_n)$, and the utility of a state
$s$ is the sum of the utilities of its facts, $\utility(s) = \underset{\fact \in s } \sum \utility(\fact)$.
A plan $\pi$ is optimal if there is no other plan $\pi'$ such that $\utility(\pi') > \utility(\pi)$.
Finally, a \define{soft goal} is any fact $\fact \in F$ with $\utility(\fact) > 0$. 

\section{Cost-Bounded Plan Reductions}
\label{sec:problem-definition-and-models}
The task to solve is to extract a subsequence of a given plan that is executable and respects the cost limit:

\begin{defn}[\textbf{Cost-Bounded Plan Reduction}]
  \label{defn:cost-bounded-plan-reduction}
  Let $\pi$ be a plan for a classical planning task $\Pi = \tuple{\V, \A, \I, \G, \cost}$ with facts $F$, a cost bound $\costbound \in \mathbb{N}_0$, and a fact utility function $\utility: F \to \mathbb{N}_0$, such that $\utility(\fact) = 0$ for each fact $\fact \in F$ if $\fact \notin \G$. A \define{Cost-Bounded Plan Reduction} of $\pi$ is a subsequence of $\pi$ that is a plan for the OSP task $\Pi_{osp} = \tuple{\V, \A, \I, \cost, \utility, \costbound}$.
\end{defn}
We denote by $\cbpr=\tuple{\Pi, \pi, \costbound, \utility}$ the problem of finding a cost-bounded plan reduction of a given plan $\pi$ for $\Pi$ subject to cost bound $\costbound$ and fact utility function $\utility$.
A solution for $\cbpr$ is optimal iff it is a plan for $\Pi_{osp}$ whose utility is maximal among all plans for $\Pi_{osp}$  that are subsequences of $\pi$.\footnote{Note that if $\costbound \geq \cost(\pi)$, then $\pi$ itself is a cost-bounded plan reduction of $\pi$ with maximum utility.}

\begin{theorem}
  \label{th:np-completeness}
  Given a cost-bounded plan reduction problem $\cbpr=\tuple{\Pi, \pi, \costbound, \utility}$,  determining whether there exists a solution with a utility of at least $M$ is NP-complete.
\end{theorem}
\begin{proof}
  Checking whether a subsequence of a given plan $\pi$ is a plan for $\Pi_{osp}$ is trivially polynomial, so the problem is in NP.
  For NP-hardness, we reduce from the classical \emph{Subset Sum} problem.
  Let  $S = \{x_1,\dots,x_n\}$ be a set of integers and $T$ a target sum.
  We construct an instance of the cost-bounded plan-reduction problem as follows.
  The planning task has a binary variable $v_i$ for each integer in $S$, actions $select_i$ that set $v_i$ to true (i.e., include $x_i$ in the sum), an initial state where all $v_i$ are false, and a goal requiring all $v_i$ to be true.
  The initial plan is $\pi = \tuple{select_1,\dots, select_n}$, with cost function $\cost(select_i) = x_i$ and fact utility function $\utility(v_i) = x_i$.
  The cost bound $\costbound$ and minimal utility $M$ are both $T$.
  Any cost-bounded plan reduction of $\pi$ whose cost does not exceed $\costbound$ and whose utility is at least $M$ corresponds exactly to selecting a subset of integers whose sum is $T$.
\end{proof}

We propose two methods to solve cost-bounded plan reduction problems: an OSP compilation and an Integer Linear Programming (ILP) formulation. We further derive a reduced ILP formulation (\ILPplus) together with preprocessing techniques that significantly decrease the model size.

\subsection{Compilation to Oversubscription Planning}

Our approach to finding a cost-bounded plan reduction of $\pi$ is to perform a search in the space of $\pi$'s subsequences using the compilation proposed by \bracketcite{salerno-et-al-jair2025}, which enables that search by deciding, for each action in the plan, whether to apply it or skip it.
But instead of encoding the problem as a classical planning task, we encode it as an OSP task. This is a simple way of introducing the cost bound and the fact utilities into the search process.
Given a cost-bounded plan reduction problem $\cbpr= \tuple{\Pi, \pi, \costbound, \utility}$ with $\Pi = \tuple{\V, \A, \I, \G, \cost} $ and $\pi = \tuple{a_1, \dots, a_n}$, we make soft goals always relevant, redefining $R_\pi$ (originally the set of facts in a precondition of the actions in the plan or in the goal) as the set of facts that appear in a precondition of the actions in the plan $\pi$ or that have positive utility:
\[
  R_\pi = \bigcup_{a_i \in \pi} \pre(a_i) \cup \{\fact \in F \mid \utility(\fact) > 0 \}
\]
We generate the OSP task $\PiMR_{osp} = \tuple{\V', \A', \I', \cost', \utility', \costbound}$, where $\V'$, $\A'$, $\I'$ and $\cost'$ are the same as for the original encoding for $\PiMR$ but considering the new definition of $R_\pi$, $\utility'(\factv{v'}{d})=\utility(\factv{v}{d})$ for all $v' \in \V' \setminus \{pos\}$, and $\utility'(\factv{pos}{d})=0$. Figure \ref{fig:example-osp-compilation} illustrates the OSP compilation with a simple example.
\begin{figure}
  \centering
  \begin{adjustbox}{width=1\columnwidth, keepaspectratio}
    \includestandalone[mode=image]{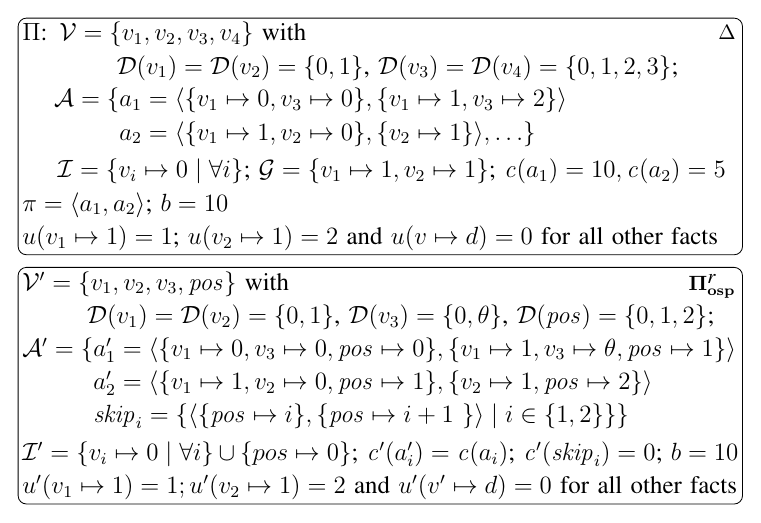}
  \end{adjustbox}
  \caption{Example of the OSP task $\PiMR_{osp}$ (bottom) for the cost-bounded plan reduction problem $\cbpr$ (top).}
  \label{fig:example-osp-compilation}
\end{figure}
\begin{theorem}
  Let $\cbpr$ be a cost-bounded plan reduction problem and let $\PiMR_{osp}$ be the corresponding OSP task. Then, the set of solutions for $\PiMR_{osp}$ is exactly the set of solutions for $\cbpr$, and a solution for  $\PiMR_{osp}$ is optimal iff it is optimal for $\cbpr$.
\end{theorem}
\begin{proof}
  By Definition~\ref{defn:cost-bounded-plan-reduction}, any solution of $\cbpr$ is a plan $\pi'$ for $\Pi_{osp}$.
  The compiled task $\PiMR_{osp}$ unfolds the original plan $\pi$ position by position and, at each step, allows either \emph{applying} or \emph{skipping} the corresponding action, preserving its relevant preconditions and effects; cost and utility functions coincide with those of $\Pi_{osp}$.
  $\pi'$ can be translated into a plan for $\PiMR_{osp}$ by selecting \emph{apply} actions exactly at the positions of the actions in $\pi'$ and \emph{skip} actions elsewhere.
  Conversely, any plan for $\PiMR_{osp}$ induces a subsequence of $\pi$ by removing all \emph{skip} steps; the compilation guarantees that this subsequence is a plan for $\Pi_{osp}$.
  Hence, there is a one-to-one correspondence between solutions of $\PiMR_{osp}$ and solutions of $\cbpr$, including the optimal ones.
\end{proof}

\subsection{Integer Linear Programming Formulation}
To solve the cost-bounded plan reduction problem $\cbpr=\tuple{\Pi, \pi, \costbound, \utility }$, we formulate the associated OSP task
$\PiMR_{osp} = \tuple{\V', \A', \I', \cost', \utility', \costbound}$, with facts $F'$,
as a 0-1 Integer Linear Program (ILP).
As in standard linear programming models defined for classical and partial satisfaction planning \cite{vossen-et-al-ijcai1999,vandenbriel-et-al-aaai2004,vandenbriel-kambhampati-jair2005}, we use binary variables to track the truth value of each fact at every layer of the horizon, binary action variables to encode action execution, and linear constraints that impose action preconditions, propagate action effects, and enforce fact persistence (frame axioms).
However, our ILP formulation focuses on \emph{plan reduction} rather than \emph{plan synthesis}.
Thus, the horizon is fixed to the length $n$ of the given plan and the action set consists solely of the \emph{apply} and the \emph{skip} actions for each plan position. Our objective is to maximize the utility of the final state induced by the resulting subplan while respecting the cost bound.

Our model is as follows:
There are two families of binary decision variables: (1) fact variables $x_{{\fact},i}$, defined for each fact $\fact \in F'$ and layer $i \in \{0,\dots,n\}$, and indicating whether $\fact$ holds in the state at layer $i$;
and (2) action variables $y_{a,i}$, defined for each action $a \in \A'_i$, where $\A'_i = \{a_i,skip_i\}$, $\A'_i \subseteq  \A'$, and for every plan position $i \in \{1,\dots,n\}$; indicating whether action $a$ is selected at layer $i$.
Since the representation of state variables is propositional, we explicitly encode the add and delete effects of actions.
For each action $a_i$, $\adds(a_i)=\eff(a_i)$, and $\del(a_i)=\{\factv{v}{d} \in \pre(a_i) \mid \factv{v}{d'} \in \eff(a_i),\, d' \neq d\} \cup \{\factv{v}{d} \in F' \mid v \notin \vars(\pre(a_i)), \factv{v}{d'} \in \eff(a_i),\, d' \neq d\}$.

\paragraph{Objective} Maximize the utility of the final state:
\[
  \max \sum_{\substack{\fact \in F' \\ \utility'(\fact)>0}} \utility'(\fact) \cdot x_{\fact,n}
\]
\paragraph{Subject to}
\begin{compactenum}
  \small
  \item Cost bound:   $\sum_{i=1}^{n}  \sum_{a \in \A'_i} \cost'(a_i) \cdot y_{a,i} \leq \costbound$
  \item Initial state:     $x_{\fact,0} =
    \begin{cases}
      1, & \text{if } \fact \in \I', \\
      0, & \text{otherwise}
    \end{cases}
    \quad \forall \fact \in F'$
  \item Exactly one action per layer:
  \[
    y_{a_i,i} + y_{\skipa_i,i} = 1 \quad \forall i \in \{1,\dots,n\}
  \]
  \item Precond. of selected actions must be true in the previous layer:
  \[
    \begin{aligned}
      y_{a,i} & \le x_{\fact,i-1}
              &                   & \forall a \in \A'_i,\ \forall \fact \in \pre(a),\ \forall i \in \{1,\dots,n\}
    \end{aligned}
  \]
  \item Del effects of selected actions must be false at the action layer:
  \[
    \begin{aligned}
      x_{\fact,i} & \le 1 - y_{a,i}
                  &                 & \forall a \in \A'_i,\ \forall \fact \in \del(a),\ \forall i \in \{1,\dots,n\}
    \end{aligned}
  \]
  \item Add effects of selected actions must be true at the action layer:
  \[
    \begin{aligned}
      x_{\fact,i} & \ge y_{a,i}
                  &             & \forall a \in \A'_i,\ \forall \fact \in \adds(a),\ \forall i \in \{1,\dots,n\}
    \end{aligned}
  \]

  \item Facts not modified by selected actions persist:
  \[
    \begin{aligned}
      x_{\fact,i} & \ge x_{\fact,i-1} - \sum_{\substack{a \in \A'_i                                                  \\ p \in \del(a)}} y_{a,i}
                  &                                                 & \forall p \in F',\ \forall i \in \{1,\dots,n\} \\
      x_{\fact,i} & \le x_{\fact,i-1} + \sum_{\substack{a \in \A'_i                                                  \\ p \in \adds(a)}} y_{a,i}
                  &                                                 & \forall p \in F',\ \forall i \in \{1,\dots,n\} \\
    \end{aligned}
  \]

\end{compactenum}

The ILP model represents $\PiMR_{osp}$, and therefore any feasible solution for it is a solution for $\cbpr$, and any optimal solution is an optimal solution for $\cbpr$.


\section{Refined ILP Formulation}
\label{sec:refined-model}

The former ILP model is a direct translation of the OSP model. An alternative and more compact ILP formulation can be derived by considering the following observations:
(1) skip action variables are fully determined by their corresponding apply action variables, and can be dropped without altering the set of optimal cost-bounded plan reductions;
(2) the $\pos$ variable that sequences the actions along the plan is unnecessary, since the ILP already sequences actions through the layer index;
(3) in a propositional representation, it is not necessary to represent irrelevant facts, i.e., those of the form $\factv{v}{\theta}$.
In the SAS$^+$ representation, these facts are needed to prevent a variable from retaining its previous value when an action changes it and the new value is never read. In contrast, the propositional representation already ensures that the fact corresponding to the previous value is deleted;
(4) when an action modifies a variable that is not in its preconditions, the propositional encoding for delete effects does not need to eliminate every possible value of that variable; it suffices to delete only those values that may have held beforehand; and
(5) there is no need to represent facts over all plan steps, but only at the \emph{change layers} induced by actions.
So instead of tracking full state evolution at every step, we only need to explicitly represent facts initially and at points where they are modified by an action.
Between those points, facts can be treated as persistent by omission.
We derive the reduced formulation from the components of $\PiMR_{osp} = \tuple{\V', \A', \I', \cost', \utility', \costbound}$. First, we obtain $\V''$ from $\V'$ by removing the variable $\pos$ and the irrelevant value ($\theta$) from the domains of the remaining variables.
Next, we obtain $\A''$ from $\A'$ by removing all skip actions and all preconditions and effects dealing with $\pos$ from apply actions. Likewise, $\I''$ is obtained from $\I'$ by removing all facts over $\pos$ and all facts $\factv{v}{\theta}$ that assign variables an irrelevant initial value. The functions $\cost''$ and $\utility''$ are redefined accordingly, while the cost bound $b$ is left unchanged.  We denote the resulting set of facts by $F''$.
For each $a_i \in \A''$, we encode add effects as before, $\adds(a_i)=\eff(a_i)$, but redefine delete effects as follows: $\del(a_i)=\{\factv{v}{d} \in \pre(a_i) \mid \factv{v}{d'} \in \eff(a_i),\, d' \neq d\} \cup \{\factv{v}{d} \in \I'' \cup \bigcup_{j<i} \eff(a_j) \mid v \notin \vars(\pre(a_i)), \factv{v}{d'} \in \eff(a_i),\, d' \neq d\}$.
That is, when an action modifies a variable that does not appear in its preconditions, the delete effects include only those facts corresponding to alternative values that could hold at that point.

For each fact $\fact \in F''$, we define its \emph{event sequence} $\events_\fact$ and its  \emph{precondition set} $\preset_\fact$. The event sequence $\events_\fact$ is the ordered list of plan positions at which an action changes the truth value of $\fact$ (that is, adds or deletes it), together with an initial event at layer $0$ that encodes its initial truth value. Formally, $\events_\fact = (\event_0, \event_1, \dots, \event_k)$, such that $\event_0 < \event_1 < \dots < \event_k$, $\event_0 = 0$, and $\{\event_1, \dots, \event_k\} = \{ i \in \{1,\dots,n\} \mid \fact \in \adds(a_i) \cup \del(a_i)\}$. The precondition set $\preset_\fact$ is the set of plan positions at which $\fact$ is a precondition of an action, that is, $\preset_\fact = \{ i \in \{1, \dots,n \} \mid \fact \in \pre(a_i)\}$.
Finally, for each fact $\fact \in F''$ and each position $i>0$ with $i \in \events_\fact \cup \preset_\fact \cup \{n+1\}$, we define $\lastevent_\fact(i)$ as the \emph{last event affecting $p$ prior to $i$}; formally,
$\lastevent_\fact(i) = \max\{\event_j \in \events_\fact \mid \event_j < i\}$.
The decision variables are: (i) fact variables $x_{\fact,i}$ for each fact $\fact \in F''$ and event $i \in \events_\fact$, which indicate whether $\fact$ holds in the state at layer $i$; and (ii) action variables $y_i$ for each plan position $i \in \{1,\dots,n\}$, which indicate whether the apply action $a_i \in \A''$ is selected. The objective and the constraints are defined as follows.

\paragraph{Objective} Maximize the utility of the final state, where the utility of a fact is determined by its last event in the plan:
\[
  \max \sum_{\substack{\fact \in F'' \\ \utility''(\fact)>0}} \utility''(\fact) \cdot x_{\fact,\lastevent_\fact(n+1)}
\]
\paragraph{Subject to}
\begin{compactenum}
  \small
  \item Cost bound:
  \[
    \sum_{i=1}^{n} \cost''(a_i) \cdot y_i \le \costbound
  \]
  \item Initial state:
  \[
    x_{\fact,0} =
    \begin{cases}
      1, & \text{if } \fact \in \I'', \\
      0, & \text{otherwise}
    \end{cases}
    \quad \forall \fact \in F''
  \]
  \item Each precondition of a selected action must be true at the layer of its last event before the action:
  \[
    y_i \le x_{\fact,\lastevent_\fact(i)} \quad \forall \fact \in \pre(a_i),\ \forall i \in \{1,\dots,n\}
  \]
  \item Del effects of selected actions must be false at the action layer:
  \[
    x_{\fact,i} \le 1 - y_i \quad \forall \fact \in \del(a_i),\ \forall i \in \{1,\dots,n\}
  \]
  \item Add effects of selected actions must be true at the action layer:
  \[
    x_{\fact,i} \ge y_i \quad \forall \fact \in \adds(a_i),\ \forall i \in \{1,\dots,n\}
  \]
  \item  Each effect of an unselected action persists from the layer of its last event before the action:

  \[
    \begin{aligned}
      x_{\fact,i} & \ge x_{\fact,\lastevent_\fact(i)} - \mathbb{I}_{\{p \in \del(a_i)\}} \cdot y_i  &  & \forall i \in \{1,\dots,n\} ,               \\
                  &                                                                                 &  & \forall \fact \in \del(a_i) \cup \adds(a_i) \\
      x_{\fact,i} & \le x_{\fact,\lastevent_\fact(i)} + \mathbb{I}_{\{p \in \adds(a_i)\}} \cdot y_i &  & \forall i \in \{1,\dots,n\} ,               \\
                  &                                                                                 &  & \forall \fact \in \del(a_i) \cup \adds(a_i) \\
    \end{aligned}
  \]

  where $\mathbb{I}_{\{p \in \del(a_i)\}}$ is an indicator function that is $1$ if $p \in \del(a_i)$ and $0$ otherwise, and $\mathbb{I}_{\{p \in \adds(a_i)\}}$ is defined analogously.

\end{compactenum}

\noindent Note that, $x_{\fact,i}$ is only defined for events $i \in \events_\fact$, and that it is guaranteed to be defined for constraints 4, 5 and 6 since they only apply to facts that are modified by the action at layer $i$ (i.e., $\fact \in \del(a_i) \cup \adds(a_i)$), and therefore $i \in \events_\fact$. And also that $x_{\fact,0}$ from constraint 2 belongs to $\events_\fact$ since it is the initial event for $\fact$.

\section{Experiments}
\label{sec:experiments}
To evaluate our approaches, we use the OSP versions of the 14 domains of the IPC~2011
created by \bracketcite{garciaolaya-et-al-aij2021}.
We generate initial plans for every instance with the Fast Downward planner\footnote{Available at \url{https://www.fast-downward.org}.}, using the \textit{seq-sat-fdss-2018} portfolio with a time limit of 30 minutes per instance and requesting a single plan.
For 14 problems, no plan was found within the time limit: 10 in {\floortile}, 1 in {\tidybot} and 3 in {\transport}.
These problems were discarded, since we require a valid initial plan as input.
We create problems with different degrees of oversubscription, with cost bounds of \(25\%\), \(50\%\), and \(75\%\) of their solution cost.
As in \bracketcite{garciaolaya-et-al-aij2021}, we consider two utility schemes: one where all goals have the same value (\emph{homogeneous}), and another where each goal is assigned an integer utility in \([1,10]\) (\emph{heterogeneous}).
All experiments were run on an Intel Xeon X3470 (2.93\,GHz) under Ubuntu~20.04.6 with 30\,GB of RAM and a wall-clock limit of 30 minutes per run. All benchmarks, code, and data are available online~\cite{deltoro-et-al-zenodo2026}.
For the OSP approach, we use \emph{Sym-Osp}~\cite{speck-katz-aaai2021}, a state-of-the-art optimal oversubscription symbolic planner.
For the ILP formulations, we use CPLEX~12.10\footnote{Available at the IBM CPLEX Optimization Studio website: \url{https://www.ibm.com/products/ilog-cplex-optimization-studio}.}, in its default multi-threaded configuration ($\text{mt}$) and also with a single thread ($\text{st}$).

We evaluate three approaches:
(i)~\OSP, using \emph{Sym-Osp};
(ii)~\ILP, the base ILP formulation; and
(iii)~\ILPplus, the refined ILP formulation presented in Section~\ref{sec:refined-model}.
All approaches start from the PDDL input and require its translation and compilation into $\PiMR_{osp}$; Figure~\ref{fig:encoding-total-time} includes this preprocessing step.

\begin{table*}[t]
  \centering
  \small
  \setlength{\tabcolsep}{3pt}
  \begin{tabular}{l rrrr rrrr rrrr}
    \toprule
    \multirow{2}{*}{Domain}
                     & \multicolumn{4}{c}{25\%}
                     & \multicolumn{4}{c}{50\%}
                     & \multicolumn{4}{c}{75\%}                                                                          \\
    \cmidrule(lr){2-5} \cmidrule(lr){6-9} \cmidrule(lr){10-13}
                     & \OSP                     & \ILPmt & \ILPst & \ILPplusst
                     & \OSP                     & \ILPmt & \ILPst & \ILPplusst
                     & \OSP                     & \ILPmt & \ILPst & \ILPplusst                                           \\
    \midrule
    barman (20)      & 17                       & 20     & 20     & 20           & 14 & 20 & 20 & 20 & 13 & 20 & 20 & 20 \\
    elevators (20)   & 4                        & 20     & 20     & 20           & 2  & 20 & 20 & 20 & 2  & 20 & 20 & 20 \\
    floortile (10)   & 10                       & 10     & 10     & 10           & 10 & 10 & 10 & 10 & 10 & 10 & 10 & 10 \\
    nomystery (20)   & 20                       & 20     & 20     & 20           & 20 & 20 & 20 & 20 & 20 & 20 & 20 & 20 \\
    openstacks (20)  & 2                        & 15     & 20     & 20           & 0  & 15 & 20 & 20 & 0  & 15 & 20 & 20 \\
    parcprinter (20) & 11                       & 20     & 20     & 20           & 8  & 20 & 20 & 20 & 8  & 20 & 20 & 20 \\
    parking (20)     & 20                       & 20     & 20     & 20           & 20 & 20 & 20 & 20 & 20 & 20 & 20 & 20 \\
    pegsol (20)      & 20                       & 20     & 20     & 20           & 20 & 20 & 20 & 20 & 20 & 20 & 20 & 20 \\
    scanalyzer (20)  & 20                       & 20     & 20     & 20           & 20 & 20 & 20 & 20 & 20 & 20 & 20 & 20 \\
    sokoban (20)     & 20                       & 20     & 20     & 20           & 20 & 20 & 20 & 20 & 20 & 20 & 20 & 20 \\
    tidybot (19)     & 18                       & 19     & 19     & 19           & 18 & 19 & 19 & 19 & 18 & 19 & 19 & 19 \\
    transport (17)   & 1                        & 17     & 17     & 17           & 1  & 17 & 17 & 17 & 0  & 17 & 17 & 17 \\
    visitall (20)    & 2                        & 5      & 9      & 20           & 1  & 5  & 9  & 20 & 1  & 5  & 9  & 20 \\
    woodworking (20) & 4                        & 20     & 20     & 20           & 2  & 20 & 20 & 20 & 2  & 20 & 20 & 20 \\
    \midrule
    Total (\textbf{266})
                     & 169                      & 246    & 255    & \textbf{266}
                     & 156                      & 246    & 255    & \textbf{266}
                     & 154                      & 246    & 255    & \textbf{266}                                         \\
    \bottomrule
  \end{tabular}
  \caption{Coverage per domain, cost bound percentage, and approach for heterogeneous utilities. Each group of four columns corresponds to a cost bound (25\%, 50\%, 75\%). Within each group, columns report the number of instances solved by \OSP, \ILPmt (multi-threaded base ILP), \ILPst (single-threaded base ILP), and \ILPplusst (single-threaded reduced ILP), respectively. The number in parentheses after each domain name indicates the total number of instances.}
  \label{tab:coverage}
\end{table*}
Table~\ref{tab:coverage} reports the coverage results for heterogeneous utilities.
Comparing \ILPmt (multi-threaded base ILP) and \ILPst (its single-threaded counterpart) reveals that restricting CPLEX to a single thread recovers 9 instances: all 5 in {\openstacks} and 4 of the 11 missing in {\visitall}. In all cases these failures are timeouts: multi-threaded CPLEX introduces synchronization overhead and explores the branch-and-bound tree differently, which on these large instances results in less coverage than single-threaded search. For this reason we only include \ILPplusst (single-threaded reduced ILP) for the reduced ILP formulation.

All ILP configurations outperform \OSP across all budget levels, but under the most restrictive budget (25\%) the gap is less pronounced, indicating that \OSP behaves comparatively better under severe budget reductions.
For the homogeneous case, similar patterns hold: \OSP solves 177, 161, and 158 instances at 25\%, 50\%, and 75\% respectively, while \ILPst solves 255 and \ILPplusst solves all 266 instances.
Each ILP configuration solves the same number of instances regardless of the budget level, whereas \OSP exhibits significant variation (169 vs.\ 154 from 25\% to 75\%).
In ILP, the model structure (variables and constraints) is fixed for a given instance, and changing the budget only modifies the cost bound constraint.
By contrast, in OSP, the budget directly affects the set of possible plans: higher bounds allow longer subplans and many more feasible combinations, which can make search harder.
Similar effects have been reported before: ILP model size scales with plan length~\cite{kautz-walser-ker2000,vandenbriel-kambhampati-jair2005}, and symbolic OSP can require substantially more state expansions as the cost bound increases~\cite{speck-katz-aaai2021}.

\OSP never solves instances with plans much longer than 450 actions, whereas \ILPst handles plans up to about 700 actions and \ILPplusst solves all instances, including the 11 {\visitall} instances on which \ILPst times out.
These instances have plans of over 1000 actions (up to 3537).

Beyond plan length, key factors limiting \OSP coverage are goal interaction and the number of goals.
Domains such as {\elevators}, {\transport}, {\parcprinter}, and {\woodworking} exhibit tightly coupled subgoals. {\openstacks} combines strong goal interactions with a large number of goals, while {\visitall} is challenging primarily due to its high goal count.

  \begin{figure*}[t!]
    \centering
    \begin{subfigure}[b]{0.27\textwidth}
      \centering
      \includegraphics[width=\linewidth]{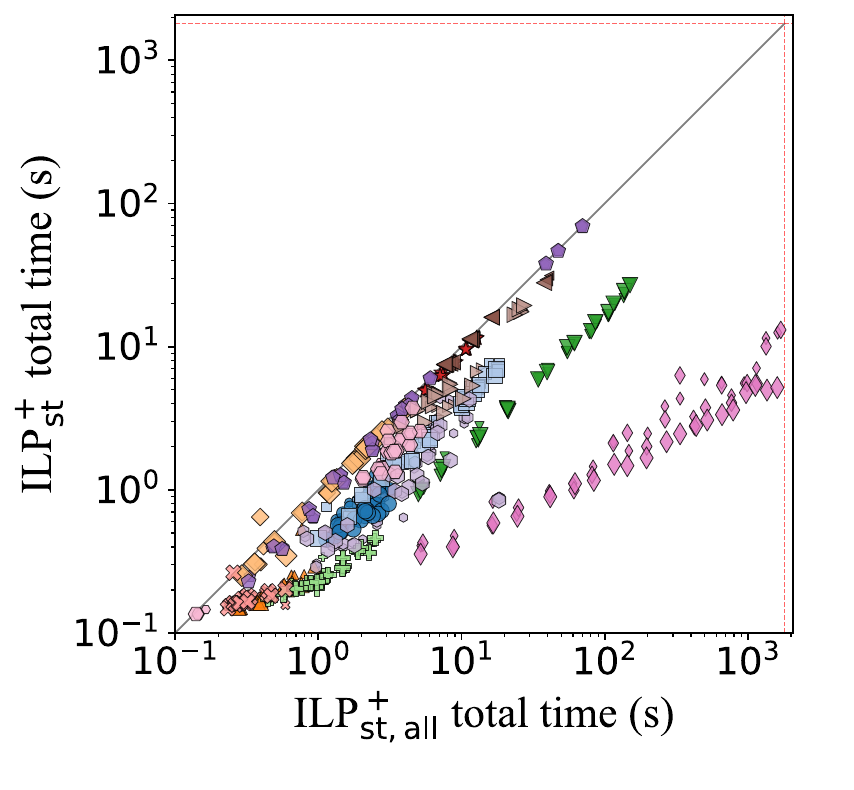}
      \caption{Total time.}
      \label{fig:encoding-total-time}
    \end{subfigure}\hspace{0.012\textwidth}%
    \begin{subfigure}[b]{0.27\textwidth}
      \centering
      \includegraphics[width=\linewidth]{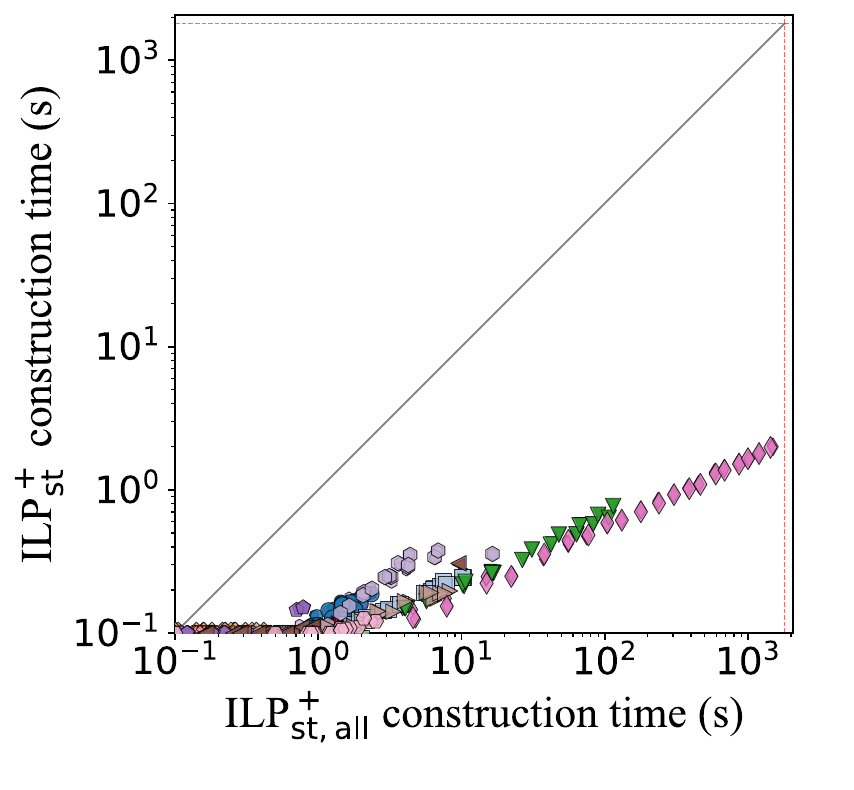}
      \caption{Construction time.}
      \label{fig:encoding-construction}
    \end{subfigure}\hspace{0.012\textwidth}%
    \begin{subfigure}[b]{0.27\textwidth}
      \centering
      \includegraphics[width=\linewidth]{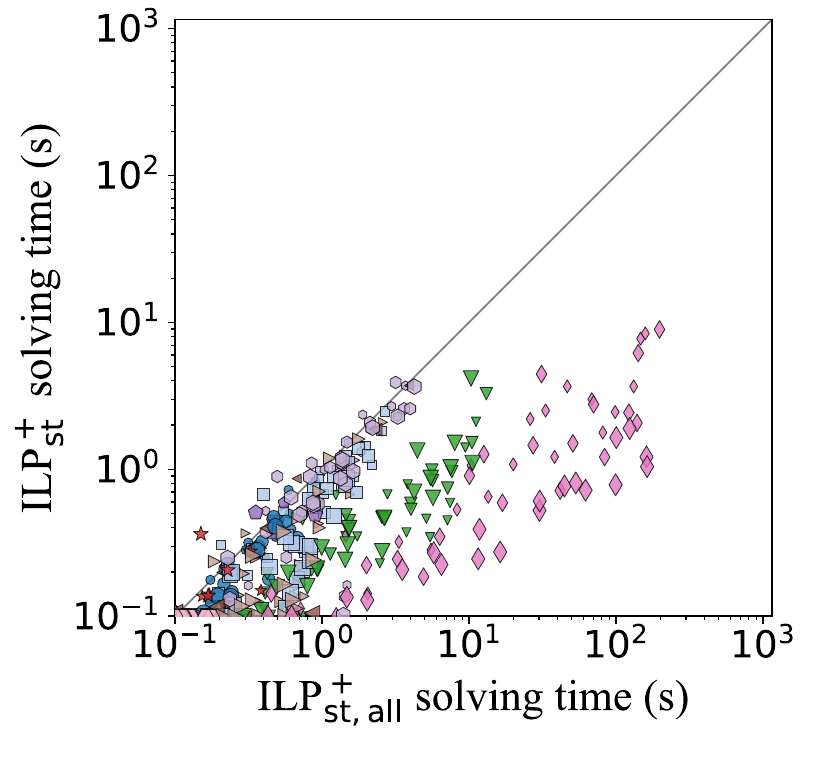}
      \caption{Solving time.}
      \label{fig:encoding-solving}
    \end{subfigure}\hspace{0.012\textwidth}%
    \newsavebox{\scatterbox}%
    \savebox{\scatterbox}{\includegraphics[width=0.27\textwidth]{figures/scatter_encoding_total_time_by_domain_util10_paper.pdf}}%
    \begin{subfigure}[b]{0.14\textwidth}
      \centering
      \includegraphics[height=\ht\scatterbox]{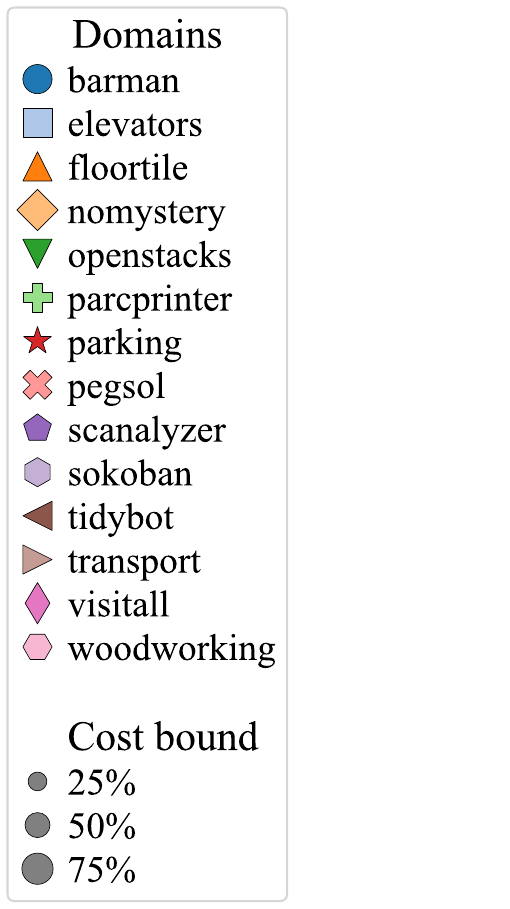}
      \caption*{}
    \end{subfigure} 

  \caption{Scatter plots comparing \ILPplusstall (fact variables at every layer) vs.\ \ILPplusst (fact variables only at change layers) for (a)~total time, (b)~construction time, and (c)~solving time, over all domains and cost bounds with heterogeneous utilities. Both variants use single-threaded CPLEX. Points below the diagonal indicate faster runtimes with event sequences at change layers.}
  \label{fig:encoding-comparison}
\end{figure*}

As introduced in Section \ref{sec:refined-model}, the reduced model ($\ILPplus$) differs from the base model ($\ILP$) in several ways. Its implementation also includes code-level optimizations for model construction. Leaving aside these implementation optimizations, the main source of the speedup of $\ILPplus$ over $\ILP$ is that it represents facts only at their change layers (event sequences), rather than at every plan step, which substantially reduces the number of variables and constraints in the ILP model. Other changes, such as dropping skip actions, have a smaller impact on the overall speedup.
To isolate the effect of considering event sequences, we compare $\ILPplusstall$, which uses the refined formulation but introduces all plan steps in the event sequence of each fact, creating fact variables at every plan layer, against $\ILPplusst$.
They differ only in whether the event sequences contain all plan positions or only the change layers.
Using event sequences at change layers reduces the number of variables by 86--99.5\% and the number of constraints by 83--99.5\% across domains.
Figure~\ref{fig:encoding-comparison} compares total runtime, model-construction time, and CPLEX solving time, for all cost bound percentages and heretogeneous utilities. Across the 798 instances solved by both encodings, the overall median ILP speedup is 6.3x, with per-domain medians ranging from 2.8x ({\scanalyzer}) to 116.1x ({\visitall}). The speedup in model-construction time is even larger: 10.5x overall, with per-domain medians from 3.6x to 268.7x. And the solving-time speedup is 2.2x overall, ranging from 1.4x to 21.0x. This shows that using the event sequences substantially reduces the model-construction time and also the solving time, particularly for larger instances.

\section{Conclusion}
\label{sec:conclusion}
We have formalized and optimally solved the cost-bounded plan-reduction problem by introducing two methods: an OSP compilation and an ILP formulation.
We further propose a refined ILP formulation  that significantly reduces the model size.
Experiments show that the ILP formulations outperform the OSP one across all budget levels, and that the refined ILP model introduced in this paper is much faster than the base ILP model.

As future work, we plan to further improve the ILP formulation by exploring additional model reductions, 
to integrate the method with satisficing OSP approaches based on heuristic goal selection to bring plans that exceed the cost bound back within it~\cite{garciaolaya-et-al-aij2021}, and to incorporate explainability capabilities.

\section*{Acknowledgements}

This work was partially supported by MICIU/AEI/10.13039\linebreak/501100011033, by ``ERDF A way of making Europe'', and by ``ESF+'', under grants PID2021-127647NB-C21 and PRE2022-104392.

\bibliography{abbrv, literatur, crossref, extra}

\end{document}